\documentclass[twoside,11pt]{article}
\usepackage{jair, theapa, rawfonts}
\usepackage{amsthm,amsmath}
\usepackage{aliascnt,cleveref}
\usepackage[noend]{algorithmic}
\usepackage{algorithm}
\usepackage[usenames]{color} % Only used in comment commands
\usepackage{xcolor}
\usepackage{url}
\usepackage{graphicx}
\newcommand{\red}[1]{{\color{red} #1}}

\ShortHeadings{PP MAP with Provable Guarantees}
{Beimel and Brafman}
\firstpageno{1}

\crefname{algorithm}{Algorithm}{Algorithms}
\crefname{clm}{Claim}{Claims}

\crefname{notation}{Notation}{Notations}
\crefname{theorem}{Theorem}{Theorems}
\crefname{claim}{Claim}{Claims}
\crefname{proposition}{Proposition}{Propositions}
\crefname{definition}{Definition}{Definitions}

\crefname{cor}{Corollary}{Corollaries}

\crefname{obs}{Observation}{Observations}
\crefname{example}{Example}{Examples}
\crefname{construct}{Construction}{Constructions}

\crefname{question}{Question}{Questions}

\crefname{remark}{Remark}{Remarks}
\renewcommand{\Cref}[1]{\cref{#1}}
\definecolor{ToDoColor}{rgb}{0.1,0.2,1}

\newcommand{\mafs}{{\sc{mafs}}}
\newcommand{\smafs}{{\sc{secure mafs}}}
\newcommand{\masas}{\textsc{mad-a*}}
\newcommand{\mastrips}{\textsc{ma-strips}}
\newcommand{\pp}{\textsc{pp-mas}}
\newcommand{\strips}{\textsc{strips}}

\newcommand{\pre}{\mathrm{pre}}
\newcommand{\eff}{\mathrm{eff}}
\newcommand{\PST}{\mathrm{PST}}
\newcommand{\ST}{\mathrm{ST}}
\newcommand{\algname}[1]{\mbox{\sc #1}}
\newcommand{\Sim}{\algname{Sim}}
\newcommand{\view}{\operatorname{\rm view}}

\newcommand{\set}[1]{\{#1\}}
\newtheorem{lemma}{Lemma}
\newtheorem{claim}{Claim}

\newtheorem{definition}{Definition}
\newtheorem{theorem}{Theorem}

\newcommand{\commentout}[1]{}

\begin{document}

\title{Privacy Preserving Multi-Agent Planning with Provable Guarantees}

\author{\name Amos Beimel \email beimel@cs.bgu.ac.il\\
       \name Ronen Brafman \email brafman@cs.bgu.ac.il \\
       \addr Ben-Gurion University of the Negev,\\
       Be'er Sheva, Israel}

% For research notes, remove the comment character in the line below.
% \researchnote

\maketitle

\begin{abstract}
In privacy-preserving multi-agent planning, a group of agents attempt to cooperatively solve a multi-agent planning problem while maintaining private their data and actions. Although much work was carried out in this area in past years,
its theoretical foundations have not been fully worked out. Specifically, although algorithms with precise privacy guarantees
exist~\cite{Yao82b,GMW87}, even their most efficient implementations are not fast enough on realistic instances,
%they are of theoretical interest only, 
whereas for practical algorithms no meaningful privacy guarantees exist.
\smafs~\cite{Brafman15}, a variant of the multi-agent forward search algorithm~\cite{nissim2014distributed}  
is the only practical algorithm to attempt to offer more precise guarantees, but only in very limited settings and with proof sketches only. 
In this paper we formulate a precise notion of secure computation for search-based algorithms and prove that \smafs\ has this property in all domains. We also provide a proof of its completeness.
\end{abstract}

\section{Introduction}
\label{Introduction}

As our world becomes better connected and more open ended, and autonomous agents are no longer  science fiction, a need arises for enabling groups of  agents to cooperate in generating a plan for diverse tasks that none of them can perform alone, in a cost-effective manner. Indeed, much like ad-hoc networks, one would expect various contexts to naturally lead to the emergence of ad-hoc teams of agents that can benefit from cooperation. Such teams could range from groups of manufacturers teaming up to build a product that none
can build on their own, to groups of robots sent by different agencies or countries to help in disaster settings. To perform complex tasks, these agents need to combine their diverse skills effectively. 
Planning algorithms can  help  achieve this goal.

Most planning algorithms require full information about the set of actions and state variables in the domain. However, often, various aspects of this information are private to an agent, and it is not eager to share them. For example, the manufacturer is eager to let everyone know that it can supply motherboards, but it will not want to disclose the local process used to construct them, its suppliers, its inventory level, and the identity of its employees. 
Similarly, rescue forces of country A may be eager to help citizens of country B suffering from a tsunami, but without having
to provide detailed information about the technology behind their autonomous bobcat to country B, or to country C's humanoid evacuation robots. In both cases, agents have public capabilities  they are happy to share, and private processes and information that support these capabilities, which they prefer (or possibly require) to be kept private.

With this motivation in mind, a number of algorithms have recently been devised for distributed privacy-preserving planning~\cite{Bonisoli14,FMAP14,LB14,NBJAIR}. In these algorithms,  agents supply a public interface only, and through a distributed planning process, come up with a plan that achieves the desired goal without being required to share a complete model of their actions and local state with other agents. But there is a major caveat: it is well known from the literature on secure multi-party computation~\cite{Yao82b} that the fact that a distributed algorithm does not require an agent to {\em explicitly\/} reveal private information does not imply that other agents cannot deduce such private information from other information communicated during the run of the algorithm. Consequently, given that privacy is the raison-d'etre for these algorithms, it is important to strive to improve the level of privacy provided, and to provide formal guarantees of such privacy properties.

To the best of our knowledge, to date, there have been two attempts to address this issue. In~\cite{TozickaSK17}, the authors describe
a secure planner for multi-agent systems. However, as they themselves admit, this planner is impractical, as it requires computing all possible solutions. \cite{Brafman15} describes \smafs\, a modification of the {\sc multi-agent forward search} algorithm~\cite{nissim2014distributed} in which an agent never sends similar states. \smafs\ is an efficient algorithm. In fact, an implementation of it based on an equivalent macro sending technique~\cite{MaliahSB16} shows state of the art performance.  But it is not clear what security guarantees
it offers. While~\cite{Brafman15} provides some privacy guarantees, they are restricted to very special cases, and it seems
most plausible that \smafs\ is not secure in general.

The goal of this paper is to place the development of \smafs\ on firm footing by developing appropriate notions of privacy 
that are useful and realizable in the context of search algorithms, to characterize the privacy preserving properties
of \smafs\ and to provide rigorous proofs for its correctness and completeness. 
We define a notion of $\beta$-indistinguishable secure computation, and more specifically, we suggest a notion of PST-secure computation which is not as strong as that of strong privacy,
but is meaningful and more stringent than weak privacy. Roughly speaking, given a function $\beta$ on planning instances,
we say that an algorithm is $\beta$-indistinguishable if it will send the same messages during computation for any two
instances whose $\beta$ value is identical. PST-secure computation refers to the special case in which $\beta$ returns
a projected version of the search space -- one in which only the value of public variables is available.

The paper is structured as follows: First, we describe the basic model of privacy-preserving classical multi-agent planning.
Then, we discuss some basic notions of privacy. Next, we gradually develop more practical versions of PST-secure planning
algorithms, eventually describing an algorithm that is, essentially \smafs, and prove that the latter is sound and complete, and
is PST-secure.

\section{The Model}
\label{model}

\mastrips~\cite{Brafman200828}  is a minimal extension of \strips\ to multi-agent domains.
A \strips\ problem is a  4-tuple $\Pi = \langle P,A,I,G \rangle$, where 
\begin{itemize}
\item
$P$ is a finite set of primitive propositions, which are essentially the state variables; a {\em state\/} is a truth assignment to $P$.
\item
$I$ is the initial state. 
\item
$G$ is the set of goal states. 
\item
$A$ is a set of actions.
%usually described by the subset of propositions in $P$ that are assigned the value {\em true\/}.  $I\subseteq P$ is
%the initial state, i.e., $p$ is assigned {\em true} in the initial state iff $p\in I$. $G\subseteq P$ is the set goals. It is not necessarily a single state,
%and $s$ is a goal state iff $s\models G$.
Each action $a$ has the form $a=\langle \pre(a),\eff(a) \rangle$, where  $ \pre(a)\subset P$ is the set of preconditions of $a$
and  $ \eff(a)$ is a set of literals, denoting the effects of action $a$. We use $a(s)$ to denote the state attained by applying $a$ in $s$. The state $a(s)$ is well defined iff $s\models \pre(a)$. In that case, $a(s)\models p$ (for $p\in P$) iff $p\in \eff(a)$ or $s\models p$ and $\neg p\not\in \eff(a)$.
\end{itemize}
A {\em plan\/} $\pi = a_1,\ldots, a_m$ is a solution to $\Pi$ iff $a_m(\cdots a_1(s)\cdots)\models G$.

An \mastrips\ problem is  a \strips\ problem in which the action set $A$ is partitioned among a set $\Phi=\{\varphi_i\}_{i=1}^{k}$ of
agents.  Formally, $\Pi = \langle P,\{A_i\}_{i=1}^{k},I,G \rangle$, where $P,I,G$ are as above, and $A_i$ is the set of actions of $\varphi_i$.

Work on privacy-preserving multi-agent planning seeks algorithms that generate good, or possibly optimal plans while not disclosing private information about their actions and the variables that they manipulate. For this to be meaningful, one has to first define what information is private and what information is not. Here we focus on the standard notion of private actions and private propositions. Thus, each action $a_i\in A_i$ is either {\em private} to agent $\varphi_i$ or {\em public}. Similarly, each proposition $p$ is either private to some agent $\varphi_i$ or public. To make sense, however, $p$ can be private to agent $\varphi_i$ {\em only\/} if $p$ does not appear in the description of an action $a_j\in A_j$ for $j\neq i$. Similarly, $a_i$ can be
private to  $\varphi_i$ only if all propositions in $a_i$'s preconditions are either public or private to $\varphi_i$
and all propositions in $a_i$'s 
effects are private to $\varphi_i$.

Hence, a {\em privacy preserving \mastrips\ problem} (\pp) is defined by as a set of
local planning problems:
$\Pi=\{\Pi_i : i=1,\ldots,k\}$
where 
$\Pi_i = \langle P_i^{\mathrm{prv}},P^{\mathrm{pub}},A_i^{\mathrm{prv}},A_i^{\mathrm{pub}},I_i,I^{\mathrm{pub}},G \rangle$. Here, $I^{\mathrm{pub}}$ is the value of $P^{\mathrm{pub}}$ in the initial state, and
the
goal is shared among all agents and
involves public propositions only. Furthermore, any action $a\in A_i^{\mathrm{prv}}$  involves private propositions only. 
We use $A_i$ to denote $A_i^{\mathrm{prv}}\cup A_i^{\mathrm{pub}}$.
A solution for a \pp\ problem is the sequence of all the public actions in a solution for the
\mastrips\ problem. 

We note that  a more refined notion of privacy was suggested in~\cite{,}. While we believe that the ideas discussed in this paper
can be extended to this setting, we leave this for future work.

Recall that in classical planning, we assume that the world state is fully observable to the acting agent and actions are deterministic. The multi-agent setting shares these assumptions, except that full observability is w.r.t.~the primitive propositions in 
$P_i^{\mathrm{prv}}\cup P^{\mathrm{pub}}$.

An issue that often arises is whether private goals should be allowed, or should all goals be public. Public goals make it easier for all agents to detect goal achievement, and have been assumed in most past work. As there is a simple reduction from private to public goals, albeit one that makes public the fact that all private goals of an agent have been achieved, we will maintain the assumption that all goal propositions are public.

Next, we define the notion of a \emph{public projection}. The \textit{public projection} $\pi_{\mathrm{proj}}(a)$ of an action
$a\in A_i$, $a=\langle \pre(a),\eff(a) \rangle$, is defined
as  $\pi_{\mathrm{proj}}(a)=\langle \{ p\in P^{\mathrm{pub}} | p\in\pre(a) \},\{\ell\in P^{\mathrm{pub}} | \ell\in\eff(a)\} \rangle$. That is, the same action, but with its private propositions removed.
Accordingly, $\pi_{\mathrm{proj}}(a)$ for $a\in A_i^{\mathrm{prv}}$ is empty. 
The \textit{public projection} $\pi_{\mathrm{proj}}(s)$ of a state is the partial assignment obtained by projecting $s$ to $P^{\mathrm{pub}}$.

Now, we define 
% the public projection of $\Pi_i = \langle P_i^{prv},P^{pub},A_i^{prv},A_i^{pub},I_i,I^{pub},G \rangle$ to be
% $\pi_{proj}\Pi_i = \langle P^{pub},\{\pi_{proj}(a) | a\in A_i^{pub}\},\pi_{proj}(I_i),G \rangle$. And the 
$\pi_{\mathrm{proj}}(\Pi)$, the
public projection of 
$\Pi=\{\Pi_i : i=1,\ldots,k\}$ to be 
the \strips\ planning problem:
$\langle P^{\mathrm{pub}},\{\pi_{\mathrm{proj}}(a):a\in A^{\mathrm{pub}}_i,1\leq i \leq k\},I^{\mathrm{pub}},G\rangle$.

%todo Example. Ronen.

The \emph{search-tree} induced by a planning problem plays a key role in our definition of privacy in distributed forward search planning.

\begin{definition}
The search tree associated with an MA planning problem $\Pi=\langle P,\{A_i\}_{i=1}^{k},I,G \rangle$, denoted by $\ST(\Pi)$, is a tree inductively defined below, where 
every node is labeled by a state and is either private to some agent or public, and every edge is labeled by an action.  
The root is labeled by $I$, and is public. The children of a node $v$ labeled by a state $s$ are defined as follows:
\begin{itemize}
\item
If $v$ is public, then for every $a$ applicable in $s$ there is a child labeled by  $a(s)$. %(Duplicate states are ignored).   
\item
If $v$ is private to $\varphi_i$, then for every $a\in A_i$ applicable in $s$ there is a child labeled by  $a(s)$.
\item In both cases, the node $a(s)$ is public if $a$ is public, and $a(s)$ is  private to $\varphi_i$
if $a$ is private to $\varphi_i$.
\item The edge from $s$ to $a(s)$ is labeld by $a$.
\end{itemize}
We will also assume the existence of some lexicographic ordering over states which defines
the order of the children of a node. We assume that public variables appear before private variables in this order.
\end{definition}

Next we define a concept of the public projection of a search tree. First, we project all states into their public parts.
Then, we connect every public node to its closest public descendants, remove all private nodes, and remove duplicate
children in the resulting tree. Formally:
\begin{definition}
The \emph{public-projection of the search
tree of $\Pi$} (denoted $\PST(\Pi)$) is a  tree, defined below, whose nodes are labeled by assignments to the public variables of $\Pi$
and edges are labeled by public actions. 
Each node in $\PST(\Pi)$ corresponds to a list of public nodes in the search-tree $ST(\Pi)$, where the public states of all the nodes in the list are the public state of the node in $\PST(\Pi)$ (this list is used only to construct $\PST(\Pi)$ from $\ST(\Pi)$ and is not part of $\PST(\Pi)$).
The tree is inductively defined. 
\begin{itemize}
\item
The root of $\PST(\Pi)$ 
corresponds to the root of $\ST(\Pi)$ and is labeled by $I^{\mathrm{pub}}$. 
\item
Let $w$ be a node in $\PST(\Pi)$, with public state $s$, that corresponds to  public nodes $v_1,\dots,v_k$ in the search tree $\ST(\Pi)$. Denote the (public and private) states of $v_1,\dots,v_k$ by $s_1,\dots,s_k$ respectively. We define the children of  $w$
in two stages:
\begin{itemize}
\item
First, for every $i\in\set{1,\dots,k}$ and every public descendants  $v'$ of $v_i$ such that 
all internal nodes in the path from $v_i$  to $v'$ are private, i.e., 
the labels of the edges on the path from $v_i$ to $v'$ are actions $a_1,\ldots,a_\ell$ such that $a_1,\ldots,a_{\ell-1}$ are private actions and $a_\ell$ is a public action,
we construct a child $w'$.
We label the edge from $w$ to $w'$ by the last actions on this path, namely, by $a_\ell$.
The public state of $w'$ is the public state in $a_k(\cdots a_2(a_1(s_i)))$ and we associate 
$v'$ to $w'$.
\item
 We remove duplicated children. That is, if $w_1$ and $w_2$ are children of $w$ such that the actions labeling the edges $(w,w_1)$ and $(w,w_2)$ are the same and the public states of $w_1$ and $w_2$ are the same, then we merge $w_1$ and $w_2$ and associate all the nodes associated to them to the merged node. We repeat this process until there are no children that can be merged. 
\end{itemize}
\end{itemize}
\end{definition}

\section{Privacy Guarantees}
The main property of interest from a solution algorithm to
a \pp\ planning problem, aside from soundness and completeness, is the level of privacy it preserves.
The main privacy-related question one asks regarding a \pp\ algorithm  is whether coalitions of agents participating in the planning algorithm will be able to gain information about the private propositions and actions of other agents.

In what follows we work under the following assumptions:
\begin{itemize}
\item Agents are {\em honest, but curious\/}. This is a well known assumption in secure multi-party computation (see, e.g.,~\cite{LindellP10}).
According to this assumption, which we believe applies to many real-world interactions among business partners and ad-hoc teams,
the agents perform the algorithm as specified, but are curious enough to collude and try to learn what they can about the other agents without acting maliciously. (Alternatively, consider malicious agents that  eavesdrop on the communication among agents, but are not part of the team, so they cannot intervene.) 
%\item There are no private communication channels (\rnote{why do we care about this}
\item The algorithm is synchronous. That is agent operate with a common clock, and send messages in rounds and
these messages are immediately delivered without corruption or delay.
\item Perfect security, that is, even an unbounded adversary cannot learn any additional information beyond the leakage function (defined below).
\end{itemize}

To date, most work was satisfied with algorithms that never explicitly expose private information, typically by encrypting this information prior to
communicating it to other agents. Consequently, we say that an algorithm is {\em weakly private\/} if the names of private actions and
private state variables and their values are never communicated explicitly.

However, the fact that information is not explicitly communicated is not sufficient. Consider, for example an algorithm in which agents share with each other their complete domains, except that the names of private actions and state variables are obfuscated by (consistently) replacing each with some
arbitrary random string. This satisfies the requirement of weak privacy, but provides the other agents with a complete model that is isomorphic to the real model. For example, imagine a producer who expects exclusivity from its suppliers. With this scheme, the producer will not know the real names of other customers of its suppliers,
but it will certainly learn of their existence. Similarly, a shipping company may not want to have others learn about the size of its fleet, or the number of workers it employs.

At the other extreme we have {\em strong privacy\/}. We say that an algorithm is {\em strongly private\/} if no coalition of agents can deduce from the information
obtained during a run of this algorithm any information that it cannot deduce from the public projection of the planning problem,
the private information the coalition has (i.e., the initial states and the actions of the agents in the coalition), and
the public projection of ``its solution''. As we are considering search problems, where many 
solutions can exists, the traditional privacy definition for functions does not apply. The problem is that the solution chosen by the algorithm can leak information (e.g., 
an algorithm that returns the lexicographically first solution leaks no previous solutions exists). See \cite{BeimelCNW08} for a discussion on this problem and a suggestion of a definition of privacy for search problems.

Furthermore, strong privacy is likely to be very difficult to achieve and to prove unless stronger cryptographic methods are introduced. With
the latter, it will be
possible to develop algorithms that are strongly private, but, at least
with our current knowledge, this is likely to come at substantial computational cost that will render them not practical for the size of inputs we would like to consider.
Weak privacy, on the other hand, seems too weak in most cases, and provides no real guarantee, as it is not clear what information is deducible from the algorithm. 

Given this state of affairs, where in the existing algorithms strong privacy is not as practical as desired, whereas weak privacy tells us little, if anything, about the information that might be leaked, it is
important to provide tools that will specify the privacy guarantees of existing and new algorithms.
Here we would like to suggest a type of privacy ``lower-bound'' in the form of an indistinguishability guarantee. More specifically, given a function $\beta$ defined on planning domains, we say that an algorithm is \emph{$\beta$-indistinguishable},
if a coalition of agents participating in the planning algorithm  solving a problem $\Pi$ cannot distinguish between the current domain and any other domains $\Pi'$ such that $\beta(\Pi) = \beta(\Pi')$. We provide two equivalent definitions of privacy. 

We define the view of the of a set of agents $T$, denoted $\view_T(x)$, 
in an execution of a deterministic algorithm with inputs $x=(x_1,\dots,x_n)$
as all the information it sees during the execution,
namely,  %the public information $\pi_{\mathrm{proj}}(\Pi)$, 
the inputs of the agents in $T$ (namely, $(x_i)_{i\in T}$)
and the messages exchanged during the execution of the algorithm.

\begin{definition}
\label{def:ind-dist}
Let $\beta:\set{0,1}^*\rightarrow \set{0,1}^*$ be a (leakage) function. We say that a deterministic algorithm is \emph{$\beta$-indistinguishable} if for every set $T$ of agents and for every two inputs
$x=(x_1,\dots,x_n)$ and $y=(y_1,\dots,y_n)$ such that $x_i=y_i$ for every $i \in T$ and $\beta(x)=\beta(y)$ the view of $T$ is the same,
i.e., $\view_T(x)=\view_T(y)$.
\end{definition}

\begin{definition}
\label{def:ind-sim}
Let $\beta:\set{0,1}^*\rightarrow \set{0,1}^*$ be a (leakage) function. We say that a deterministic algorithm is \emph{$\beta$-indistinguishable} if there exists a simulator $\Sim$ such that for every set $T$ of agents and for every input
$x=(x_1,\dots,x_n)$ the view of  $T$ is the same as the output of the simulator that is given $(x_i)_{i\in T}$ and $\beta(x)$, i.e., $\Sim(T,(x_i)_{i\in T},\beta(x))=\view_T(x)$.
\end{definition}

In \cref{def:ind-sim}, the simulator is given the inputs of the agents in $T$ and $\beta(x)$ -- the output of the leakage function applied to the inputs of all agents. The simulator is required to produce all the messages that were exchanged during the algorithm.
If such simulator exists, then all the information that the adversary can learn from the execution of the algorithm is implied by the inputs of the parties in $T$ and $\beta(x)$.

\begin{claim}
The two definitions are equivalent.
\end{claim}
\begin{proof}
Assume that an algorithm is $\beta$-indistinguishable according to \cref{def:ind-sim}. 
Let 
$x=(x_1,\dots,x_n)$ and $y=(y_1,\dots,y_n)$ be two inputs
such that $x_i=y_i$ for every $i \in T$ and $\beta(x)=\beta(y)$.
Thus, 
%\begin{equation}
$\Sim(T,(x_i)_{i\in T},\beta(x))=\Sim(T,(y_i)_{i\in T},\beta(y)).$
%\end{equation} 
Therefore, by \cref{def:ind-sim},
$\view_T(x)=\Sim(T,(x_i)_{i\in T},\beta(x))=\Sim(T,(y_i)_{i\in T},\beta(y))=\view_T(y)$.

Assume that an algorithm is $\beta$-indistinguishable according to \cref{def:ind-dist}. Let 
$x=(x_1,\dots,x_n)$ be any input. We define a simulator for the algorithm.
Given $T,(x_i)_{i\in T},\beta(x)$ we construct a simulator $\Sim$ as follows:
\begin{itemize}
\item
Finds inputs $(y_i)_{i \notin T}$
such that $\beta(y)=\beta(x)$, where $y_i=x_i$ if $i \in T$.
\item
Outputs $\view_T(y)$. 
\end{itemize}
By \cref{def:ind-dist}, $\view_T(x)=\view_T(y)$, thus,
$\Sim(T,(x_i)_{i\in T},\beta(x))=\view_T(x)$, as required in \cref{def:ind-sim}.
\end{proof}
%\anote{We should choose one of the above two definitions.}

Note that the simulator is not given the output of the function computed by the algorithm, information that is implied by the messages exchanged in the algorithm. 
The simulator can compute the view of $T$, hence the output, from the information it gets.
This implies that the leakage $\beta(x)$ (together with $(x_i)_{i\in T}$) determines the output of the algorithm. This is an important feature of our definition, as we consider search problems where there can  be many possible outputs. The output that an algorithm returns might leak information on the inputs (see~\cite{BeimelCNW08}), and it is not clear how to compare the privacy provided by two algorithms returning different solutions. Our definition bypasses this problem as it explicitly specifies the leakage.   

In this paper, we will focus on a particular function $\beta$ that returns the public projection of the
problem's search tree. That is, the algorithms we will consider will have the property that a set of agents
cannot distinguish between two problem instances whose public projection and their PST are identical.
We will refer to this as {\em PST-indistinguishable security}.

A recently proposed example of privacy w.r.t.\ a class of domains is {\em cardinality preserving privacy}~\cite{MaliahSS17} where the idea is that agents cannot 
learn the number of values of a some variable, such as the number of locations served by a track.
(Defining this formally requires using multi-valued variable domains.)
Another notion of privacy recently introduced is \emph{agent privacy}~\cite{FaltingsLP08} in which agents are not aware
of other agents with whom they do not have direct interactions -- i.e., agents that require or affect
some of the variables that appear in their own actions.  This notion is more natural when such interactions
are explicitly modelled using the notion of subset-private variables~\cite{Bonisoli14}.
These notions seem more ad-hoc and weaker than our definition of privacy.
 We will not discuss these notions in this paper.

\section{A PST-Indistinguishable Algorithm}

The goal of this section is to show that \smafs\ is PST-Indistinguishable. We will do it by gradually refining a very simple (and inefficient) algorithm to obtain an algorithm that is essentially identical to \smafs, which, as shown by~\cite{MaliahSB16}, is quite efficient in practice, and thus the first algorithm to be both practical and have clear theoretical guarantees. 
This gradual progression will make the proofs and ideas simpler.

\subsection{A Simple Algorithm}
We start with a very simple algorithm, which we shall call PST-Forward Search.
The algorithm simply constructs $\PST(\Pi)$ -- the public-projection of the search
tree of $\Pi$. 
The search progresses level by level in the public-projection of the  search tree. In a given level of the tree, each agents $\varphi_i$: (1) computes the children of all the nodes in $\PST(\Pi)$, where a child of a node results from a sequence of private actions followed by a single public action by the agent, 
and (2) sends the public state of each child (as well as a description of the path to the child) to all other agents (removing duplicates). 
The PST-Forward Search algorithm is described in \cref{alg:simple-search}. 
In this algorithm, the agents maintain a set $Q_{d}$ for every level $d$ in the tree, 
which will contain all nodes in level $d$.
Every element in the set is a node represented as a pair $(\vec{s},\vec{a})$, where $\vec{s}=(s_0,\dots,s_m)$ is a sequence of public states such that $s_0=I^{\mathrm pub}$ and $\vec{a}=(a_1,\dots,a_m)$ is a sequence of public actions.
Such a pair describes a path in the PST from the root to the node in level $d$.
To find the actions that an agent can apply from a node, it needs to compute the possible private states of that node,
as this information is not contained in the message it received. To do this\red{,} the agent reconstructs its private
state, as described in Algorithm {\bf compute-private-states}. This is, of course, highly inefficient, but has the
desired privacy property.

%\vspace{-5pt}
%\begin{algorithm}
%\caption{Simple Private Search Algorithm}
    %\label{alg:simple-search}
%\begin{algorithmic}[1]  %[n] = every n'th line is numbered
%\STATE Each agent holds a copy of the same prefix $T$ of $\PST(\Pi)$. 
%\STATE Initialization: Each agent holds the root of the tree labeled by $I^{\mathrm{pub}}$.
%\WHILE{goal has not been achieved}
 %\STATE Let $v$ be the next node in a BFS tour of $T$ and $s$ be its state. 
 %\FOR{$i=1$ \TO $n$ }
  %\STATE Agent $\varphi_i$ reconstructs its private state $s_i$ in $v$ starting from $I_i$ and updating it according to its actions on the path from the root to $v$.
	%\FORALL{sequences $a_1,\dots,a_k$ in $A_i$ applicable from $s,s_i$, where $a_1\dots,a_{k-1}$ are private and $a_k$ is public}
		%\STATE Agent $\varphi_i$  computes $s',s'_i \gets a_k(a_{k-1}(\cdots a_1(s,s_i)))$.
		%\STATE Agent $\varphi_i$ sends $a_1,\dots,a_k,s'$ to all other agents.
		%\STATE All agents add a new node $w$ as a child of $v$ in the tree. 
		%\STATE The state of $w$ is $s'$ and the label of the edge between $v$ and $w$ is $a_1,\dots,a_k$. 
		%\STATE If $s'$ satisfies the goal, then output the actions on the path from the root to $w$ and halt.
	%\ENDFOR
	%\ENDFOR
%\ENDWHILE
%\end{algorithmic}
%\end{algorithm}

\begin{algorithm}
\caption{PST Forward Search}
    \label{alg:simple-search}
\begin{algorithmic}[1]  %[n] = every n'th line is numbered
%\STATE Each agent holds a copy of the same prefix $T$ of $\PST(\Pi)$. 
\STATE {\bf initialization:} $d \gets 0$; for $i \in \set{1,\dots,n}$ set $Q_{0}=\set{(I^{\mathrm{pub}},\epsilon)}$.\\
// $Q_d$ will contain the states at level $d$ of the PST. Each agent maintains a copy of it.
%Each agent holds the root of the tree labeled by $I^{\mathrm{pub}}$.
\WHILE{goal has not been achieved}
   \STATE $d\gets d+1$; for every $i \in \set{1,\dots,n}$ agent $\varphi_i$ sets $Q_{d}\gets\emptyset$ and $C_i\gets \emptyset$.
   \FOR{$i=1$ \TO $n$ }
	    \STATE Agent $\varphi_i$ does the following:
      \FORALL{$(\vec{s},\vec{a})\in Q_{d-1}$}
         \STATE let $s$ be the last state in $\vec{s}$.
         \STATE executes $PS\gets $\textbf{ compute-private-states}$(i,\vec{s},\vec{a})$.
				
  	     \FORALL{private state $ps \in PS$}
				\FORALL{sequence $a_1,\dots,a_\ell$ of actions of $\varphi_i$ applicable from $s,ps$, where $a_1\dots,a_{\ell-1}$ are private and $a_\ell$ is public }
		        \STATE computes $(s',ps') \gets a_\ell(a_{\ell-1}(\cdots a_1((s,ps))))$ and $C_i \gets C_i \cup \set{((\vec{s},s'),(\vec{a},a_\ell))}.$
		     
	       \ENDFOR
				\ENDFOR
	    \ENDFOR
			\STATE sends $C_i$ to all agents (where the elements of $C_i$ are sent according to some canonical order).
		  \STATE each agent $\varphi_j$ updates its copy: $Q_{d}\gets Q_{d} \cup C_i$.
      \IF{the last state $s'$ in some $((\vec{s},s'),(\vec{a},a_\ell)) \in C_i$ satisfies the goal}
				\STATE	all agents output $(\vec{a},a_\ell)$ and halt. 
			\ENDIF
	\ENDFOR
\ENDWHILE
\end{algorithmic}
\end{algorithm}

\begin{algorithm}
\caption{compute-private-states$(i,\vec{s}=(s_0,\dots,s_m),\vec{a}=(a_1,\dots,a_m))$}
    \label{alg:CompPrivate}
\begin{algorithmic}[1]
\STATE\COMMENT{The algorithm reconstructs the possible private states of agent $\varphi_i$
			  starting from $I^{\rm pub},I_i$ and updating it according to the states is $\vec{s}$ and the actions of $\varphi_i$ in $\vec{a}$.} 
\STATE let $PS_0\gets\set{I_i}$.
   \FOR{$j=1$ \TO $m$}
	 \IF{$a_j$ is not an action of $\varphi_i$}
	    \STATE{$PS_j\gets PS_{j-1}$.}
			\ELSE
			\STATE $PS_j\gets \emptyset$.
			\FORALL{$ps \in PS_{j-1}$ and  sequence of private actions 
			$a'_1,\dots,a'_{\ell}$ in $A_i$ such that $a'_1,\dots,a'_{\ell},a_j$ is applicable from $s_{j-1},ps$}
			  \STATE let $(s',ps') \gets a_j(a'_{\ell}(\cdots a'_1((s_j,ps))))$.
				\red{
				\IF{$s'=s_j$}
					\STATE{$PS_j\gets PS_j\cup \set{ps'}$.}
				\ENDIF
				}
			\ENDFOR
	 \ENDIF
   \ENDFOR
	\RETURN{$PS_m$.}
\end{algorithmic}
\end{algorithm}
%Initially all agents start with a single element in their open list: the public
%projection of the initial state and the empty action sequence. 
%At each step, the agent 1) removes the first element from its open list, 2) generates all private states consistent with this projected state and its associated action sequence (see below), 3) for every projected state
%and private state consistent with it, it expands this state using all possible actions sequences containing and ending with a single
%public action, 4) projects the resulting states back to their public part and augments their action sequence
%with the public action is performed, 5) orders the projected state based on some lexicographic order. Having done this to all elements of its open list, it sends the resulting state/actions pairs and updates its open list with them.

%To generate all private states consistent with the project state and the action sequence (step 2), the agent recreates its past work. That is, it takes the initial state, applies the prefix of relevant actions
%in the sequence up to, but not including, its first action. At this point, it recreates the work it did to
%apply this first action, which could be diverse sequences of private actions. Using these sequences it recreates the (possibly many) private states obtained after applying the  action sequence. It continues with this simulation process until the end of the associated action sequence, and now has all the possible values of its private state that could occur with this action sequence.

In  \cref{alg:simple-search}, the messages sent correspond exactly to the PST nodes, 
and therefore, two domains with an identical PST will yield identical messages. To enable an exact simulation, we need to specify the order in which each agent sends the possible sequences of children in a given level; we assume that this is done in some canonical order. We supply the formal proof of privacy in the next claim.
\begin{claim}
\cref{alg:simple-search}, the simple private search algorithm, is a PST-indistinguishable secure algorithm.
\end{claim}
\begin{proof}
The simulator, given the PST $T$, traverses the tree level by level, 
in each level $d$ it goes over all agents $\varphi_i$ starting from $\varphi_1$ and ending at 
$\varphi_n$, and for each agent $\varphi_i$ it sends the nodes of level $d$ resulting from an action of $\varphi_i$,  where for each node it sends the public states and the actions on the path from the root to the node. The order of sending the nodes is as in the algorithm, according to the fixed canonical order.  
\end{proof}

\subsection{Using IDs}

%Next, observe that since the expansion is level by level, an agent will receive at some point,
%all nodes at the same level of the PST. Suppose that they generated the same public projection
%$s_p$ from two different siblings. They need not send this twice, but instead, can send $s_p$
%with all the associated action sequences. Again, there is no problem reproducing the different cases,
%from these sequences. Thus, we can allow agents to send the same projected state with a set
%of action sequences. 
%
%\rnote{It seems to me that this would work also in the following case. The agent generated $s_p$ 
%at level $k$. Later it generates $s_p$ again at level $k'>k$. It can now associate with
%$s_p$ the sequences corresponding to the states at level $k$ and at level $k'$. This is a function
%of the tree. However, if we assume that the agents first completes a level
%of the tree, sends everything and then continue, then it is not clear how this could occur.
%}

Next, we present an optimization of \cref{alg:simple-search}, which eliminates the need to compute private states, and merges some nodes in the tree, reducing the communication complexity of the algorithm. 
We call this version: PST-ID Forward Search.

%We next explain how we can avoid computing the possible private states for each node. 
Notice that only actions of $\varphi_i$ change the local state of $\varphi_i$. 
There are two approaches to use this observation.
In one approach,  for each node that is sent, the agent sending the node can locally keep  a list containing its possible local states in that node. When an agent wants to  compute  the children of some node, it looks for its last action in the path to the node and retrieves its possible local states after that action. 
%
%Consider the path information associated with each node in the PST sent in \cref{alg:simple-search}.
%To recreate its local state, the agent needs only to know its own actions in the sequence and the pubic states before each such state is executed --
%the actions of the other agents affect the public variables only, and it already knows their  values.
%Thus, instead of associating action sequences with each project state, we can associate a vector of action sequences,
%one per each agent. Since each agent cares only about its component of this vector, it can replace this action
%sequence with a unique ID -- maintaining the correspondence between the ID and its actions in its memory.
%Again, throughout, PST-indistinguishability is maintained, as the information sent is a function of the
%the PST. 
In the second approach, which we use, each agent associates the possible local states with a unique id and keeps the possible local states associated with this id.  Each time  an agent sends a node in the tree, it sends the public state of the node as well as the $n$ ids, encoding the local states of each agent. Notice that each id is not a function of these local states, but only of the particular PST node with which it is associated. When an agent wants to compute the children of a node resulting from its actions, it does the following: 
\begin{itemize}
\item
It retrieves all private states associated with its id in this state.  
\item
It expands the public state  and each possible private state using all possible actions sequences containing and ending with a single public action.
\item
For every  node reached, it generates a new id and associates
with it its local state in the states generated with this projected state,
keeping the ids of all other agents associated with the original node.
\item 
It orders the nodes based on some lexicographic order.
\item
It sends these nodes,  with their public states and their associated ids, in this order to all agents. 
\end{itemize}

Note that the above algorithm sends at each stage a vector consisting of a public state and an id for
each agent. As this id encodes the private state(s) of the agent, we can think of the message as representing the state, with its private components encoded.  The agent does not need to send neither the actions leading to the new node nor the father of the new node. Furthermore,
if two (or more) children  of a node have the same public state, the agent does not need to send them twice;
it can send one public state, together with the ids of the other agents taken from the original node, and one new id for the agent associated with  all its possible private states associated with any one of these children.  We go one step further,  merging all nodes generated by an agent in level $d$ (possibly with different fathers) if they have the same public state and the same ids for all other agents.

The formal description of the algorithm appears in \cref{alg:less-simple-search}.
The algorithm that recovers a solution after the goal has been reached is described in \cref{alg:recover-solution}.
 
\begin{algorithm}
\caption{PST-ID Forward Search}
    \label{alg:less-simple-search}
\begin{algorithmic}[1]  %[n] = every n'th line is numbered
%\STATE Each agent holds a copy of the same prefix $T$ of $\PST(\Pi)$. 
\STATE {\bf initialization:} $d \gets 0$; for every $i \in \set{1,\dots,n}$ agent $\varphi_i$ sets  $id_i\gets 0$, $Q_{0}\gets \set{(I^{\mathrm{pub}},0,\dots,0)}$, and $PS_{i}[0]\gets \set{I_i}$. \\
//$PS_{i}[j]$ denotes the local states $\varphi_i$ associated with the id $j$.
%Each agent holds the root of the tree labeled by $I^{\mathrm{pub}}$.
\WHILE{goal has not been achieved}
   \STATE $d\gets d+1$; for every $i \in \set{1,\dots,n}$ agent $\varphi_i$ sets $Q_{d}\gets \emptyset$ 
	and $E_i\gets \emptyset$.
   \FOR{$i=1$ \TO $n$ }
	    \STATE agent $\varphi_i$ does the following:
      \FORALL{$(s,j_1,\dots,j_n)\in Q_{d-1}$}
            \FORALL{private state $ps\in PS_{i}[j_i]$ }
  	        \FORALL{sequence $a_1,\dots,a_\ell$ of actions of $\varphi_i$  applicable from $s,ps$, where $a_1\dots,a_{\ell-1}$ are private and $a_\ell$ is public }
		           \STATE $(s',ps') \gets a_\ell(a_{\ell-1}(\cdots a_1((s,ps))))$ 
							 \STATE $E_i\gets E_i \cup \set{(s',j_1,\dots,j_{i-1},j_{i+1},\dots,j_n,ps')}$.
						\ENDFOR
	       \ENDFOR
			\ENDFOR
			\STATE agent $\varphi_i$ sorts the elements of $E_i$, first by the public state,  then by the $n-1$ ids, and then by the private state. Let $((s^1,j_1^1,\dots,j^1_{i-1},j^1_{i+1},\dots,j_n^1,ps^1)$ 
			$\dots,(s^t,j_1^t,\dots,j^t_{i-1},j^t_{i+1},\dots,j_n^t,ps^t))$ be the sorted elements of $E_i$.
			%\STATE $C_i\gets\set{(s^1,j_1^1,\dots,j^1_{i-1},id_i,j^1_{i+1},\dots,j_n^1)}$ and $PS_{i}[id_i]\gets \set{ps^1}$.
			\FOR{$u=1$ \TO $t$}
			   \IF{$u > 1$ \AND $s^{u-1} =s^u $ \AND $(j_1^{u-1},\dots,j^{u-1}_{i-1},j^{u-1}_{i+1},\dots,j^{u-1}_n)=
				(j_1^u,\dots,j^u_{i-1},j^u_{i+1},\dots,j_n^u)$}
				    \STATE $PS_{i}[id_i]\gets PS_{i}[id_i] \cup \set{ps^u}$.
				 \ELSE
				    \STATE $id_i\gets id_i+1$.
						\STATE $C_i\gets C_i \cup \set{(s^u,j_1^u,\dots,j^u_{i-1},id_i,j^u_{i+1},\dots,j_n^u)}$ and $PS_{i}[id_i]\gets \set{ps^u}$.
				 \ENDIF
			\ENDFOR
			\STATE $\varphi_i$ sends $C_i$ to all agents (where the elements of $C_i$ are sent according to some canonical order).
		  \STATE each agent $\varphi_j$ updates: $Q_{d}\gets Q_{d} \cup C_i$.
		\ENDFOR
		\IF{the state $s$ in some element in $Q_{d}$ satisfies the goal}
				\STATE the agents execute $sol \gets $ {\bf recover-solution},   output $sol$,  and halt.
    \ENDIF 
\ENDWHILE
\end{algorithmic}
\end{algorithm}

\cref{alg:recover-solution} described below returns a solution to the planning problem, i.e., a sequence of public actions on a path from the root of the PST to a node in level $d$ that satisfies the goal. 
Clearly, this sequence of actions should be computed from the information computed by the algorithm so far.  
Furthermore, to guarantee privacy, this sequence of actions should be determined by the PST (that is, a simulator can generate it from the PST). In \cref{alg:recover-solution} we choose it in a specific way that is fairly efficient (especially, if the agents keep additional information during \cref{alg:less-simple-search}).

In \cref{alg:recover-solution},  we say that 
$s_{d'-1},j_{1,{d'-1}},\dots,j_{n,{d'-1}} \in Q_{d'-1}$ 
leads to $s_{d'},j_{1,{d'}},\dots,j_{n,{d'}} \in Q_{d'}$ by agent $\varphi_i$ if
there exist private states $ps_{d-1}\in PS_{i}[j_{d'-1}],ps_d\in PS_{i}[j_{d'}]$, and a sequence of actions $a_1,\dots,a_\ell$ of agent $\varphi_i$ such that $a_1,\dots,a_{\ell-1}$ are private and $a_\ell$ is public and 
$a_1,\dots,a_{\ell}$ are applicable from $s_{d'-1},ps_{d'-1}$ and lead to $s_{d'},ps_{d'}$.

\begin{algorithm*}
\caption{recover-solution}
    \label{alg:recover-solution}
\begin{algorithmic}[1]
\STATE let $s_d,j_{1,d},\dots,j_{n,d}$ be the first element in $Q_{d}$ that satisfies the goal.
\STATE \COMMENT{recall that all agents have a copy of $Q_d$.} 
\FOR{$d'=d$ {\bf downto } 1 }
\STATE let $\varphi_i$ be the agent performing the last action leading to $s_{d'},j_{1,{d'}},\dots,j_{n,{d'}}$.
\STATE agent $\varphi_i$ finds the first element  $s_{d'-1},j_{1,{d'-1}},\dots,j_{n,{d'-1}} \in Q_{d'-1}$ 
leading to $s_{d'},j_{1,{d'}},\dots,j_{n,{d'}}$.
\STATE let $a_{d'}$ be the last action in  a sequence of actions 
leading from  $s_{d'-1},j_{1,{d'-1}},\dots,j_{n,{d'-1}}$ to $s_{d'},j_{1,{d'}},\dots,j_{n,{d'}}$
(if there is more than one such action, choose the lexicographically first action).
\STATE agent $\varphi_i$ sends $s_{d'-1},j_{1,{d'-1}},\dots,j_{n,{d'-1}}$ and $a_{d'}$ to all other agents. 
\ENDFOR
\RETURN $a_1,\dots,a_d$.
\end{algorithmic}
\end{algorithm*}

\begin{claim}
\label{c:less-simple}
\cref{alg:less-simple-search}, the PST-ID Forward Search algorithm, is a PST-indistinguishable secure algorithm.
\end{claim}
\begin{proof}
We construct a simulator proving that \cref{alg:less-simple-search} is a PST-indistinguishable secure algorithm. We first supply a high level description of the simulator.
The simulator, given the PST $T$, traverses the tree level by level and simulates the algorithm. For some level $d$, it goes over the agents from agent $\varphi_1$ to $\varphi_n$ and for each agent it produces a list $C_i$ as the agent would have sent, using the nodes in 
level $d$ resulting from an action of $\varphi_i$. Recall that each element in $C_i$ is a public state and a list of $n$ ids. To produce these ids (and to know which nodes should be merged), for every vertex $w$ in level $d$  the simulator computes a label, denoted by $L(w)$, that contains $n$ ids; this label is computed using the label of the father of a node $w$, denoted by $f(w)$. The labels of $w$ and $f(w)$ are the same except for the $i$th id, which is carefully computed to simulate \cref{alg:less-simple-search}.
After reaching  the first level in which there is a node satisfying the goal, the simulator, using the PST tree, reconstructs the solution that  
%\Crefrange{alg:less-simple-search}{alg:recover-solution} return.
\Cref{alg:recover-solution} returns.

The simulator is formally described  in \cref{sim:less-simple-search}.
The input in \cref{sim:less-simple-search} is a PST $T$\red{;} we denote its root  by $root$.
It can be easily proved by induction that the simulator computes the same messages as \cref{alg:less-simple-search}.
\end{proof}

\begin{algorithm}
%\floatname{algorithm}{Simulator}
\caption{Simulator for \cref{alg:less-simple-search} -- The  PST-ID Forward Search Algorithm}
    \label{sim:less-simple-search}
\begin{algorithmic}[1]  %[n] = every n'th line is numbered
%\STATE Each agent holds a copy of the same prefix $T$ of $\PST(\Pi)$. 
\REQUIRE A PST tree $T$
\STATE {\bf initialization:} $d \gets 0$; for every $i \in \set{1,\dots,n}$  set  $id_i\gets 0$, $Q_{0}\gets ((I^{\mathrm{pub}},0,\dots,0))$.
%Each agent holds the root of the tree labeled by $I^{\mathrm{pub}}$.
\WHILE{goal has not been achieved}
   \STATE $d\gets d+1$; for every $i \in \set{1,\dots,n}$ set $Q_{d}\gets \emptyset$,
	  $C_i\gets \emptyset$, and $\tilde{E}_i\gets \emptyset$.
	 \STATE $L[root]=(0,\dots,0)$.
   \FOR{$i=1$ \TO $n$ }
      \FORALL{node $w$ in level $d$ s.t.~the edge $(f(w),w)$ is labeled by an action of $\varphi_i$}
			\STATE let $L(f(w))=(j_1,\dots,j_n)$ and $s'$ be the state of node $w$.
      \STATE $\tilde{E}_i\gets \tilde{E}_i \cup \set{(s',j_1,\dots,j_{i-1},j_{i+1},\dots,j_n,w)}$.
			\ENDFOR
		\STATE sort the elements of $\tilde{E}_i$, first by the public state, then by the $n-1$ ids, and then by $w$. 
		\STATE let $((s^1,j_1^1,\dots,j^1_{i-1},j^1_{i+1},\dots,j_n^1,w^1),$ 
			$\dots,(s^t,j_1^t,\dots,j^t_{i-1},j^t_{i+1},\dots,j_n^t,w^t))$ be the sorted elements of $\tilde{E}_i$.
		%\STATE $C_i\gets\set{(s^1,j_1^1,\dots,j^1_{i-1},id_i,j^1_{i+1},\dots,j_n^1)}$.
		\FOR{$u=1$ \TO $t$}
			   \IF{ $u=1$ \OR $s^{u-1} \neq s^u $ \OR $(j_1^{u-1},\dots,j^{u-1}_{i-1},j^{u-1}_{i+1},\dots,j^{u-1}_n) \neq
				(j_1^u,\dots,j^u_{i-1},j^u_{i+1},\dots,j_n^u)$}
			    \STATE $id_i\gets id_i+1$.
					\STATE $C_i\gets C_i \cup \set{(s^u,j_1^u,\dots,j^u_{i-1},id_i,j^u_{i+1},\dots,j_n^u)}$.
				 \ENDIF
				\STATE $L(w^u)\gets (j_1^u,\dots,j^u_{i-1},id_i,j^u_{i+1},\dots,j_n^u)$.
		\ENDFOR
		\STATE send $C_i$ on behalf of $\varphi_i$ to all agents (where the elements of $C_i$ are sent according to some canonical order).
		\STATE for every $j \in \set{1,\dots,n}$ set $Q_{d}\gets Q_{d} \cup C_i$.
	\ENDFOR
	\IF{ the state $s$ in some element in $Q_d$ satisfies the goal}
				\STATE execute $sol \gets $ {\bf sim-recover-solution},   output $sol$,  and halt.
  \ENDIF 
\ENDWHILE
\end{algorithmic}
\end{algorithm}

\begin{algorithm*}
\caption{sim-recover-solution}
    \label{sim:recover-solution}
\begin{algorithmic}[1]
\STATE let $s_d,j_{1,d},\dots,j_{n,d}$ be the first element in $Q_{d}$ that satisfies the goal
and $W_d$ be all nodes $w$ in level $d$ whose public state is $s$ and whose label $L(w)$ is
$j_{1,d},\dots,j_{n,d}$.
%\STATE \COMMENT{As $Q_{1,d}=\cdots = Q_{n,d}$, all agents know these values.} 
\FOR{$d'=d$ {\bf downto } 1 }
\STATE let $F_{d'-1}\gets \set{f(w)|w \in W_{d'}}$.
\STATE let $s_{d'-1},j_{1,d'-1},\dots,j_{n,d'-1}$ be the first element in
$\set{(s(v),L(v)|v \in F_{d'-1}}$ (where $s(v)$ is the public state in the node $v$). 
\STATE let $a_{d'}$ be the lexicographically first action labeling an edge from a node
$v \in F_{d'-1}$ such that  $s(v)=s_{d'-1}$ and  $L(v)=j_{1,d'-1},\dots,j_{n,d'-1}$
to a node in $W_{d'}$.
\STATE let $W_{d'-1}$ be all nodes in $F_{d'-1}$ such that  $s(v)=s_{d'-1}$,   $L(v)=j_{1,d'-1},\dots,j_{n,d'-1}$ and there exists an edge from them to a node in $W_{d'}$ labeled by the action $a_{d'}$.
\STATE Send the message $s_{d'-1},j_{1,d'-1},\dots,j_{n,d'-1}$  and $a_{d'}$.
\ENDFOR
\RETURN $a_1,\dots,a_d$.
\end{algorithmic}
\end{algorithm*}

\subsection{Merging More Nodes}
In PST-ID Forward Search, an agent merged two nodes if they were in the same level, they had the same public state, and the ids of the other agents were the same. The simple case when two nodes were merged is if they had the same parent and there were two sequences of actions ending with the same public state (if the last action in these sequences is the same, then they are already merged in the PPT). There are somewhat more complicated scenario when nodes are merged. For example, suppose that in some public state $s$ and private state 
$ps$ in level $d$, agent $\varphi_i$ can apply two sequences of public actions $a_1,a_2$ and $a_3,a_4$, and both sequences result in the same public state $s'$. Then, the resulting two nodes are in the same level $d+2$ and they will be merged. However, suppose that also action $a_5$ is applicable in the state $s,ps$ and results in state $s'$ (in level $d+1$). The resulting node is not in the same level and the previous nodes are not merged with the new node.
As a result, the current algorithm will send two nodes that are identical in every respect, except for its id.
One key motivation for the original \smafs\ algorithm was to prevent this situation and never send two nodes
that differ only in the private state of the sending agent.

There is a simple (though probably inefficient) way of overcoming this. For this observe that, under the 
assumption that an agent will never send two states that differ only in its own id, the only way two
states $s',s''$ generated by an agent $\varphi_i$ can be identical is if they have a common ancestor $s$, 
and $s'$ and $s''$ were generated by applying actions of $\varphi_i$ only. As in the above example, these
could be sequences containing different numbers of public actions, and hence at different levels of the PST.
However, once a public action is applied by some other agent $\varphi_j$, its id will change, and hence $s'$ and $s''$ will
differ on $\varphi_j$'s id. Given this observation, it is easy to modify PST-ID FS to have the property that
an agent $\varphi_i$ never sends two nodes that are identical in all but (possibly) its id, which we call
PST-ID-E Forward Search. Whereas in PST-ID FS an agent will send each state obtained by applying
exactly one public action, in PST-ID-E, the agent expands the entire local sub-tree below a node in its open list.
That is, it will consider state reachable by applying more than one (of its) public actions. This could be a large 
sub-tree, of course, but under the assumption that all variables have finite-domains, it is finite and with appropriate
book keeping (maintaining a closed list) can be constructed in finite time. Thus, the only change is in line 8 of~\cref{alg:less-simple-search},
where the new line is
\begin{quote}
{\bf for each} sequence $a_1,\dots,a_\ell$ of actions of $\varphi_i$  applicable from $s,ps$, where $a_1\dots,a_{\ell-1}$ are {\em public or private} and $a_\ell$ is public {\bf do}.
\end{quote}

\begin{claim}
\label{c:psd-id-e}
The PST-ID-E Forward Search algorithm is a PST-indistinguishable secure algorithm.
\end{claim}
\begin{proof}
This follows immediately from the proof of~\cref{c:less-simple} using the following observation:
Take the PST, and add to it additional edges between every node and all its descendants that
are reachable using public actions of the same agent only. Now, use the simulator for PST-ID FS
on this modified tree.
\end{proof}

Note that given the modified tree in the proof above, it is possible to recover the original ordering by 
simply taking into account the number of public actions that were applied in the path from the initial
state to the current state.

\begin{claim}
\label{c:psd-id-e-differ}
In the PST-ID-E Forward Search algorithm, an agent never sends two states $s,s'$ that differ only in its own id.
\end{claim}
\begin{proof}
Consider two states $s,s'$ sent by an agent $\varphi_i$ during the run of the algorithm. Let the level of a state denote
the number of times (plus 1) a public action was applied in the path to this state by an agent such that this agent did not apply
the previous public action on the path.
First, assume that $s,s'$ have a common ancestor such that all actions on the paths from this ancestor to
$s$ and $s'$ are of the same agent $\varphi_i$. In this case, if they are identical in all other respects, 
an id that contains both their private states is formed, and only one state is sent.
Suppose that $s,s'$ do not have such ancestor.  Consider the sequence of states sent by agents on the paths from the root to $s$ and $s'$.
At some points, these states differ, and hence the id of the agent that sent the states will differ too. But from this point on,
the ids of all sending agent must change.
\end{proof}

\subsection{Heuristic Search}
So far, the algorithms we described expanded nodes in breadth-first manner, and followed some canonical
ordering within each level. PST-ID-E also fits this view, when levels are defined such that the level increases
only when a public action is applied by an agent who did not apply the last public action.
However, the privacy guarantees do not rest on this property. 
In principle, the PST can be traversed in any order, and all the above results are correct provided the
traversal ordering is a function of the PST only. Thus, for example, any heuristic search algorithm can be used, provided the heuristic
depends on the history of the public part of the state only, or on the current public state.
This follows trivially from the fact that a simulator that has access to the PST can simulate any such ordering.

\subsection{\smafs}

We are now ready to describe a PST-indistinguishable secure algorithm that is essentially 
a synchronous, breadth-first version of  \smafs~\cite{Brafman15}. 
\smafs\ is similar to PST-ID Forward Search (i.e., a message is sent after the application of a public action),
except that an agent never sends two states that differ only in its own private state -- in our case, its own id. 
The PST-ID-E algorithm has this property, but requires that an agent first explore its entire sub-tree.

To prevent resending identical states (modulo its own id), in \smafs\ the agent must maintain a list of states sent so far. Whenever it wishes to send a state $s$ with local state $ps$, it first checks if the state $s$ was sent before. If it was, it simply updates the id associated with $s$ to include $ps$.

\begin{figure}[ht]
\centerline{
\includegraphics[width=0.5\textwidth,height=7cm]{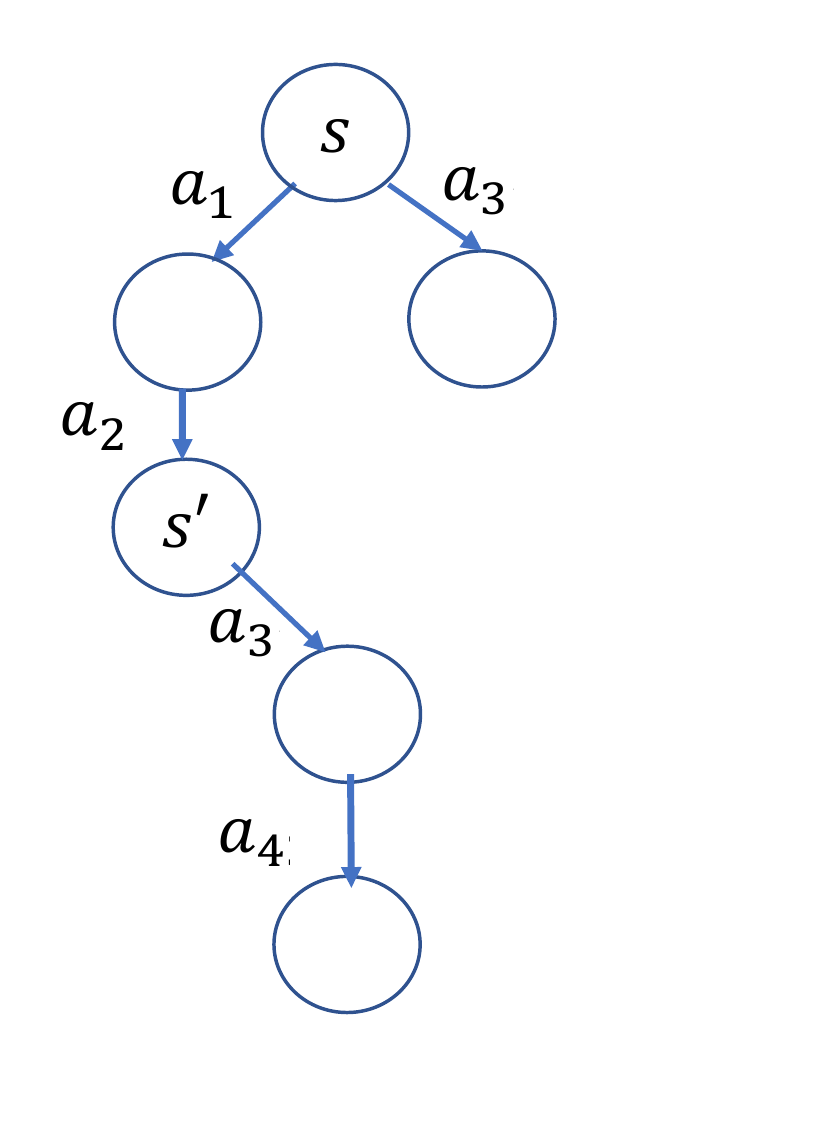}}
\caption{\label{fig:example1}An example for \smafs.} 
\end{figure}

However, this change alone is insufficient to maintain completeness. See Figure~\ref{fig:example1} for illustration of the following example. Consider some state $s$ that is
being expanded by $\varphi_i$. Suppose that the non-private part of the state is identical in  $s'=a_2(a_1(s))$ and $s$, but the local state is different. Here $a_1,a_2$ are  public actions of $\varphi_i$ that only change the private state of $\varphi_i$.
Let $a_3$ be a public action of another agent $\varphi_j$ and $a_4$ an action of $\varphi_i$. 
We claim that $\varphi_i$ may never generate $a_4(a_3(s'))$, although, as we shall see, it should. 
To see this, note that $\varphi_i$ will receive $a_3(s)$ from $\varphi_j$ and will  expand it
%Suppose that $a_6$ is not applicable in $a_5(s'')$ because of $\varphi_i$'s local state there, but that it is applicable in $a_5(s')$. 
before it generates $s'=a_2(a_1(s))$. 
Now, suppose that $a_4$ cannot be applied in $a_3(s)$ because of $\varphi_i$'s local state, but it can be applied in $a_3(s')$. Eventually, $\varphi_i$ will generate
$s'=a_2(a_1(s))$. However, it will not send it to $\varphi_j$. It will simply update the id associated with
$s$ to include the local state of $s'$. Since $s$ was already expanded, it will not attempt to re-expand it, and will miss the state $a_4(a_3(s'))$.

%However, this change alone is insufficient to maintain completeness. Consider some state $s$ that is
%being expanded by $\varphi_i$. Suppose that the non-private part of the state is identical in  $s'=a_3(a_2(a_1(s)))$ and $s''=a_4(s)$, but the local state is different. Here $a_1,a_2,a_3,a_4$ are all actions of $\varphi_i$.
%Let $a_5$ be a public action of another agent $\varphi_j$ and $a_6$ an action of $\varphi_i$. 
%We claim that $\varphi_i$ may never generate $a_6(a_5(s'))$, \red{although, as we shall see, it should}. 
%To see this, note that $\varphi_i$ will receive $a_5(s'')$ from $\varphi_j$ and will  expand it
%%Suppose that $a_6$ is not applicable in $a_5(s'')$ because of $\varphi_i$'s local state there, but that it is applicable in $a_5(s')$. 
%before it generates $s'=a_3(a_2(a_1(s)))$. 
%Now, suppose that $a_6$ cannot be applied in $a_5(s'')$ because of $\varphi_i$'s local state, but it can be applied in $a_5(s')$. Eventually, $\varphi_i$ will generate
%$s'=a_3(a_2(a_1(s)))$. However, it will not send it to $\varphi_j$. It will simply update the id associated with
%$s''$ to include the local state of $s'$. Since $s''$ was already expanded, it will not attempt to re-expand it, and will miss the state $a_6(a_5(s'))$. 

To address this issue, \smafs\ must re-expand states previously expanded when their id is modified.
Specifically, in the above example, when we modify the id of $a_2(a_1(s))$, \smafs\ will 
add $a_3(s')$ (with the appropriate ids) to a local queue and later see that $a_4$ is applicable from this state. 

\begin{algorithm}
\caption{\smafs}
    \label{alg:smafs}
\begin{algorithmic}[1]  %[n] = every n'th line is numbered
\STATE {\bf initialization:} $d \gets 0$; $Q_{0}\gets \set{(I^{\mathrm{pub}},0,\dots,0)}$; for every $i \in \set{1,\dots,n}$ agent $\varphi_i$ sets  $id_i\gets 0$, $PS_{i}[0]\gets \set{I_i}$,
and $LQ_{i,d'}\gets\emptyset$ for every $d'$.
\WHILE{goal has not been achieved}
   \STATE $d\gets d+1$; for every $i \in \set{1,\dots,n}$ agent $\varphi_i$ sets $Q_{d}\gets \emptyset$, $C_{i,d}\gets\emptyset$, %,  $LQ_{i,d+1}\gets\emptyset$, 
	and $E_i\gets \emptyset$.
   \FOR{$i=1$ \TO $n$ }
	    \STATE agent $\varphi_i$ does the following:
%ronen changed d-1 to d-2 in LQ	    
      \FORALL{$(s,j_1,\dots,j_n)\in Q_{d-1} \cup LQ_{i,d-1}$}
            \FORALL{private state $ps\in PS_{i}[j_i]$ }
						\IF{$(s,j_1,\dots,j_n)$ and $ps$ where not evaluated previously by $\varphi_i$}
  	        \FORALL{sequence $a_1,\dots,a_\ell$ of actions of $\varphi_i$  applicable from $s,ps$, where $a_1\dots,a_{\ell-1}$ are private and $a_\ell$ is public }
		           \STATE  $(s',ps') \gets a_\ell(a_{\ell-1}(\cdots a_1((s,ps))))$.
		           \IF {$(s',j_1,\dots,j_{i-1},j_{i+1},\dots,j_n,ps')$ was not generated before}
							 \STATE $E_i\gets E_i \cup \set{(s',j_1,\dots,j_{i-1},j_{i+1},\dots,j_n,ps')}$.
							 \ENDIF
						\ENDFOR
					\ENDIF
	       \ENDFOR
			\ENDFOR
			\STATE agent $\varphi_i$ sorts the elements of $E_i$, first by the public state, and then by the $n-1$ ids, and then by the private state. Let $((s^1,j_1^1,\dots,j^1_{i-1},j^1_{i+1},\dots,j_n^1,ps^1)$ 
			$\dots,(s^t,j_1^t,\dots,j^t_{i-1},j^t_{i+1},\dots,j_n^t,ps^t))$ be the sorted elements of $E_i$.
			%\STATE $C_i\gets\set{(s^1,j_1^1,\dots,j^1_{i-1},id_i,j^1_{i+1},\dots,j_n^1)}$ and $PS_{i}[id_i]\gets \set{ps^1}$.
			\FOR{$u=1$ \TO $t$}
			   \IF{there exist $d' < d$ and $id$ such that $(s^u,j_1^u,\dots,j^u_{i-1},id,j^u_{i+1},\dots,j_n^u) \in C_{i,d'}$ % Q_{d'}$
				}	
						\STATE update $PS_{i}[id]\gets PS_{i}[id] \cup \set{ps^u}$.
						%ronen -- added "was not created by i"
						\FORALL{$(s,j_1,\dots,j_{i-1},j_{i+1},\dots,j_n)$ s.t. $(s,j_1,\dots,j_{i-1},id,j_{i+1},\dots,j_n) \in Q_{d''}$ for some $d'' <d$}
%						   \STATE Updates $LQ_{i,d} \gets LQ_{i,d} \cup \bar{s}$.
%						 \FORALL {$\bar{s'}$ that is a child of $\bar{s}$ not created by $\varphi_i$}
						   \STATE \label{line:LQ} update $LQ_{i,d+(d''-d')} \gets LQ_{i,d+(d''-d')} \cup \set{(s,j_1,\dots,j_{i-1},id,j_{i+1},\dots,j_n)}$.
%						   \STATE Updates $LQ_{i,d+1} \gets LQ_{i,d+1} \cup \bar{s'}$.
%						\ENDFOR
						\ENDFOR
			   \ELSIF{$u >1$ \AND  $s^{u-1} =s^u $ \AND $(j_1^{u-1},\dots,j^{u-1}_{i-1},j^{u-1}_{i+1},\dots,j^{u-1}_n)=
				(j_1^u,\dots,j^u_{i-1},j^u_{i+1},\dots,j_n^u)$}
				    \STATE $PS_{i}[id_i]\gets PS_{i}[id_i] \cup \set{ps^u}$. //Collects ids of similar states in a level
						\label{step:PSi}
				    \ELSE  
				    \STATE update $id_i\gets id_i+1$
				    %\STATE Updates $C_{i}\gets C_{d} \cup \set{(s^u,j_1^u,\dots,j^u_{i-1},id_i,j^u_{i+1},\dots,j_n^u)}$,
						% and $PS_{i}[id_i]\gets \set{ps^u}$.
						\STATE update $C_{i,d}\gets C_{i,d} \cup \set{(s^u,j_1^u,\dots,j^u_{i-1},id_i,j^u_{i+1},\dots,j_n^u)}$,
						 and $PS_{i}[id_i]\gets \set{ps^u}$.
				    \ENDIF
			\ENDFOR
			\STATE agent $\varphi_i$ sends $C_{i,d}$ to all agents (where the elements of $C_{i,d}$ are sent according to some canonical order).
		  \STATE each agent $\varphi_j$ updates: $Q_{d}\gets Q_{d} \cup C_{i,d}$.
		\ENDFOR
		\IF{the state $s$ in some element in $Q_{d}$ satisfies the goal \label{line:finds}} 
				\STATE the agents execute $sol \gets $ {\bf recover-solution},   output $sol$,  and halt.
    \ENDIF 
\ENDWHILE
\end{algorithmic}
\end{algorithm}
The pseudo-code for \smafs\ appears in~\cref{alg:smafs}. At each level we generate a number of lists of states:
The set $C_{i,d}$ contains the new states that agent $\varphi_i$ created in level $d$; these states are sent to all agents. The set $Q_d$ contains the new states created in level $d$ by some agent,
that is $Q_d=\cup_{1\leq i \leq n} C_{i,d}$.
Furthermore, the set $LQ_{i,d}$ will contain states that were generated
by $\varphi_i$, but are not being sent because a similar state was sent earlier. These lists are initially empty.
In round $d$, each agent $\varphi_i$ expands all states in $Q_{d-1}$ and in $LQ_{i,d-1}$ using any sequence
of private actions followed by a single public action. It collects all these states into $E_d$. 
This list is sorted and all its elements are processed in order. 
For each element, agent $\varphi_i$ checks if this state did not appear before in the states it created (namely, in $C_{i,d'}$ for some $d' <d$), and if
a similar state $s'$ that differs only in $\varphi_i$'s private state appeared earlier.
If the latter is the case, let $id$ denote the id of agent $\varphi_i$ in $s'$.
We now go over all states in previous $Q_t$'s that have the id $id$, and add them to an appropriate $LQ$ list.
The list selected reflects the number of public actions that were applied to reach them from $s'$.

Observe that \smafs\ enjoys the property that an agent $\varphi_i$ will never send two states that differ only in its
own id.

\begin{algorithm}
%\floatname{algorithm}{Simulator}
\caption{Simulator for \cref{alg:smafs} -- \smafs}
    \label{sim:smafs}
\begin{algorithmic}[1]  %[n] = every n'th line is numbered
%\STATE Each agent holds a copy of the same prefix $T$ of $\PST(\Pi)$. 
\REQUIRE A PST tree $T$
\STATE {\bf initialization:} $d \gets 0$; for every $i \in \set{1,\dots,n}$  set  $id_i\gets 1$, $Q_{0}\gets ((I^{\mathrm{pub}},0,\dots,0))$.
%Each agent holds the root of the tree labeled by $I^{\mathrm{pub}}$.
\WHILE{goal has not been achieved}
   \STATE $d\gets d+1$; for every $i \in \set{1,\dots,n}$ set $Q_{d}\gets \emptyset$,
	  $C_{i,d}\gets \emptyset$, and $\tilde{E}_i\gets \emptyset$.
	 \STATE $L[root]=(0,\dots,0)$.
   \FOR{$i=1$ \TO $n$ }
      \FORALL{node $w$ in level $d$ s.t.~the edge $(f(w),w)$ is labeled by an action of $\varphi_i$}
			\STATE let $L(f(w))=(j_1,\dots,j_n)$ and $s'$ be the state of node $w$.
      \STATE $\tilde{E}_i\gets \tilde{E}_i \cup \set{(s',j_1,\dots,j_{i-1},j_{i+1},\dots,j_n,w)}$.
			\ENDFOR
		\STATE sort the elements of $\tilde{E}_i$, first by the public state, then by the $n-1$ ids, and then by $w$. 
		\STATE let $((s^1,j_1^1,\dots,j^1_{i-1},j^1_{i+1},\dots,j_n^1,w^1),$ 
			$\dots,(s^t,j_1^t,\dots,j^t_{i-1},j^t_{i+1},\dots,j_n^t,w^t))$ be the sorted elements of $\tilde{E}_i$.
		%\STATE $C_i\gets\set{(s^1,j_1^1,\dots,j^1_{i-1},id_i,j^1_{i+1},\dots,j_n^1)}$.
		\FOR{$u=1$ \TO $t$}
		 \IF{ there exists $d'< d$ and $id$ such that $(s^u,j_1^u,\dots,j^u_{i-1},id,j^u_{i+1},\dots,j_n^u) \in C_{i,d'}$}
				\STATE $L(w^u)\gets (j_1^u,\dots,j^u_{i-1},id,j^u_{i+1},\dots,j_n^u)$.
				\ELSE
			   \IF{ $u=1$ \OR $s^{u-1} \neq s^u $ \OR $(j_1^{u-1},\dots,j^{u-1}_{i-1},j^{u-1}_{i+1},\dots,j^{u-1}_n) \neq
				(j_1^u,\dots,j^u_{i-1},j^u_{i+1},\dots,j_n^u)$}
			    \STATE $id_i\gets id_i+1$.
					\STATE $C_{i,d}\gets C_{i,d} \cup \set{(s^u,j_1^u,\dots,j^u_{i-1},id_i,j^u_{i+1},\dots,j_n^u)}$.
				 \ENDIF
				\STATE $L(w^u)\gets (j_1^u,\dots,j^u_{i-1},id_i,j^u_{i+1},\dots,j_n^u)$.
			\ENDIF
		\ENDFOR
		\STATE send $C_{i,d}$ on behalf of $\varphi_i$ to all agents (where the elements of $C_{i,d}$ are sent according to some canonical order).
		\STATE for every $j \in \set{1,\dots,n}$ set $Q_{d}\gets Q_{d} \cup C_{i,d}$.
	\ENDFOR
	\IF{ the state $s$ in some element in $Q_{d}$ satisfies the goal}
				\STATE execute $sol \gets $ {\bf sim-recover-solution},   output $sol$,  and halt.
  \ENDIF 
\ENDWHILE
\end{algorithmic}
\end{algorithm}

\bibliography{cites}

\begin{thebibliography}{}

\bibitem[\protect\BCAY{Beimel, Carmi, Nissim,\ \BBA\ Weinreb}{Beimel
  et~al.}{2008}]{BeimelCNW08}
Beimel, A., Carmi, P., Nissim, K., \BBA\ Weinreb, E. \BBOP2008\BBCP.
\newblock \BBOQ Private approximation of search problems\BBCQ\
\newblock {\Bem {SIAM} J. Comput.}, {\Bem 38\/}(5), 1728--1760.

\bibitem[\protect\BCAY{Bonisoli, Gerevini, Saetti,\ \BBA\ Serina}{Bonisoli
  et~al.}{2014}]{Bonisoli14}
Bonisoli, A., Gerevini, A., Saetti, A., \BBA\ Serina, I. \BBOP2014\BBCP.
\newblock \BBOQ A privacy-preserving model for the multi-agent propositional
  planning problem\BBCQ\
\newblock In {\Bem ICAPS'14 Workshop on Distributed and Multi-Agent Planning}.

\bibitem[\protect\BCAY{Brafman}{Brafman}{2015}]{Brafman15}
Brafman, R.~I. \BBOP2015\BBCP.
\newblock \BBOQ A privacy preserving algorithm for multi-agent planning and
  search\BBCQ\
\newblock In Yang, Q.\BBACOMMA\  \BBA\ Wooldridge, M.\BEDS, {\Bem Proceedings
  of the Twenty-Fourth International Joint Conference on Artificial
  Intelligence, {IJCAI} 2015, Buenos Aires, Argentina, July 25-31, 2015},
  \BPGS\ 1530--1536. {AAAI} Press.

\bibitem[\protect\BCAY{Brafman\ \BBA\ Domshlak}{Brafman\ \BBA\
  Domshlak}{2008}]{Brafman200828}
Brafman, R.~I.\BBACOMMA\  \BBA\ Domshlak, C. \BBOP2008\BBCP.
\newblock \BBOQ From one to many: Planning for loosely coupled multi-agent
  systems\BBCQ\
\newblock In {\Bem ICAPS}, \BPGS\ 28--35.

\bibitem[\protect\BCAY{Faltings, L{\'{e}}aut{\'{e}},\ \BBA\ Petcu}{Faltings
  et~al.}{2008}]{FaltingsLP08}
Faltings, B., L{\'{e}}aut{\'{e}}, T., \BBA\ Petcu, A. \BBOP2008\BBCP.
\newblock \BBOQ Privacy guarantees through distributed constraint
  satisfaction\BBCQ\
\newblock In {\Bem Proceedings of the 2008 {IEEE/WIC/ACM} International
  Conference on Intelligent Agent Technology, Sydney, NSW, Australia, December
  9-12, 2008}, \BPGS\ 350--358.

\bibitem[\protect\BCAY{Goldreich, Micali,\ \BBA\ Wigderson}{Goldreich
  et~al.}{1987}]{GMW87}
Goldreich, O., Micali, S., \BBA\ Wigderson, A. \BBOP1987\BBCP.
\newblock \BBOQ How to play any mental game or {A} completeness theorem for
  protocols with honest majority\BBCQ\
\newblock In Aho, A.~V.\BED, {\Bem Proceedings of the 19th Annual {ACM}
  Symposium on Theory of Computing, 1987, New York, New York, {USA}}, \BPGS\
  218--229. {ACM}.

\bibitem[\protect\BCAY{Lindell\ \BBA\ Pinkas}{Lindell\ \BBA\
  Pinkas}{2009}]{LindellP10}
Lindell, Y.\BBACOMMA\  \BBA\ Pinkas, B. \BBOP2009\BBCP.
\newblock \BBOQ Secure multiparty computation for privacy-preserving data
  mining\BBCQ\
\newblock {\Bem The Journal of Privacy and Confidentiality}, {\Bem 1\/}(1),
  59--98.

\bibitem[\protect\BCAY{Luis\ \BBA\ Borrajo}{Luis\ \BBA\ Borrajo}{2014}]{LB14}
Luis, N.\BBACOMMA\  \BBA\ Borrajo, D. \BBOP2014\BBCP.
\newblock \BBOQ Plan merging by reuse for multi-agent planning\BBCQ\
\newblock In {\Bem ICAPS'14 Workshop on Distributed and Multi-Agent Planning}.

\bibitem[\protect\BCAY{Maliah, Shani,\ \BBA\ Brafman}{Maliah
  et~al.}{2016}]{MaliahSB16}
Maliah, S., Shani, G., \BBA\ Brafman, R.~I. \BBOP2016\BBCP.
\newblock \BBOQ Online macro generation for privacy preserving planning\BBCQ\
\newblock In {\Bem Proceedings of the Twenty-Sixth International Conference on
  Automated Planning and Scheduling, {ICAPS} 2016, London, UK, June 12-17,
  2016.}, \BPGS\ 216--220.

\bibitem[\protect\BCAY{Maliah, Shani,\ \BBA\ Stern}{Maliah
  et~al.}{2017}]{MaliahSS17}
Maliah, S., Shani, G., \BBA\ Stern, R. \BBOP2017\BBCP.
\newblock \BBOQ Collaborative privacy preserving multi-agent planning -
  planners and heuristics\BBCQ\
\newblock {\Bem Autonomous Agents and Multi-Agent Systems}, {\Bem 31\/}(3),
  493--530.

\bibitem[\protect\BCAY{Nissim\ \BBA\ Brafman}{Nissim\ \BBA\
  Brafman}{2014a}]{nissim2014distributed}
Nissim, R.\BBACOMMA\  \BBA\ Brafman, R.~I. \BBOP2014a\BBCP.
\newblock \BBOQ Distributed heuristic forward search for multi-agent
  planning\BBCQ\
\newblock {\Bem Journal of Artificial Intelligence Research (JAIR)}, {\Bem 51},
  293--332.

\bibitem[\protect\BCAY{Nissim\ \BBA\ Brafman}{Nissim\ \BBA\
  Brafman}{2014b}]{NBJAIR}
Nissim, R.\BBACOMMA\  \BBA\ Brafman, R.~I. \BBOP2014b\BBCP.
\newblock \BBOQ Distributed heuristic forward search for multi-agent
  planning\BBCQ\
\newblock {\Bem Journal of AI Research}, {\Bem 51}, 292--332.

\bibitem[\protect\BCAY{Torre{\~n}o, Onaindia,\ \BBA\ Sapena}{Torre{\~n}o
  et~al.}{2014}]{FMAP14}
Torre{\~n}o, A., Onaindia, E., \BBA\ Sapena, O. \BBOP2014\BBCP.
\newblock \BBOQ Fmap: Distributed cooperative multi-agent planning\BBCQ\
\newblock {\Bem Applied Intelligence}, {\Bem 41\/}(2), 606--626.

\bibitem[\protect\BCAY{Tozicka, Stolba,\ \BBA\ Komenda}{Tozicka
  et~al.}{2017}]{TozickaSK17}
Tozicka, J., Stolba, M., \BBA\ Komenda, A. \BBOP2017\BBCP.
\newblock \BBOQ The limits of strong privacy preserving multi-agent
  planning\BBCQ\
\newblock In {\Bem Proceedings of the Twenty-Seventh International Conference
  on Automated Planning and Scheduling, {ICAPS} 2017, Pittsburgh, Pennsylvania,
  USA, June 18-23, 2017.}, \BPGS\ 297--305.

\bibitem[\protect\BCAY{Yao}{Yao}{1982}]{Yao82b}
Yao, A. C.-C. \BBOP1982\BBCP.
\newblock \BBOQ Protocols for secure computations (extended abstract)\BBCQ\
\newblock In {\Bem FOCS}, \BPGS\ 160--164.

\end{thebibliography}
\bibliographystyle{theapa}

\end{document}